\documentclass[preprint,12pt]{elsarticle}

\usepackage{amssymb}
\usepackage{amsthm}
\usepackage{lineno}
\usepackage{bm}
\usepackage[colorlinks=true, pdfauthor=author]{hyperref}
\usepackage{wrapfig}
\usepackage{booktabs} 
\usepackage{algorithm}
\usepackage{algorithmic}
\usepackage{graphicx}
\usepackage{amsmath}
\usepackage{amssymb}
\usepackage{amsfonts}
\usepackage{mathptmx}
\usepackage{multirow}
\usepackage{graphicx}
\usepackage{mathrsfs}
\usepackage{colortbl}
\usepackage{subfigure}
\usepackage{indentfirst}
\usepackage{ulem}
\usepackage{url}
\usepackage{verbatim}
\usepackage{rotating}
\usepackage{adjustbox}
\usepackage{wrapfig}
\newtheorem{definition}{Definition}
\newtheorem{theorem}{Theorem}
\newtheorem{remark}{Remark}
\graphicspath{{./images/}}

\begin{document}

\begin{frontmatter}

\title{Robust-GBDT: GBDT with Nonconvex Loss for Tabular Classification in the Presence of Label Noise and Class Imbalance}

\author[first]{Jiaqi Luo}
\ead{jiaqi.luo@dukekunshan.edu.cn}
\author[first]{Yuedong Quan}
\ead{yuedong.quan@dukekunshan.edu.cn}
\author[first]{Shixin Xu \corref{cor1}}
\cortext[cor1]{Corresponding author}
\ead{shixin.xu@dukekunshan.edu.cn}
\affiliation[first]{organization={Data Science Research Center, Duke Kunshan University},
            addressline={No.8 Duke Avenue}, 
            city={Kunshan},
            postcode={215000}, 
            state={Jiangsu Province},
            country={China}}

\begin{abstract}

Dealing with label noise in tabular classification tasks poses a persistent challenge in machine learning. While robust boosting methods have shown promise in binary classification, their effectiveness in complex, multi-class scenarios is often limited. Additionally, issues like imbalanced datasets, missing values, and computational inefficiencies further complicate their practical utility.
This study introduces Robust-GBDT, a groundbreaking approach that combines the power of Gradient Boosted Decision Trees (GBDT) with the resilience of nonconvex loss functions against label noise. By leveraging local convexity within specific regions, Robust-GBDT demonstrates unprecedented robustness, challenging conventional wisdom.
Through seamless integration of advanced GBDT with a novel Robust Focal Loss tailored for class imbalance, Robust-GBDT significantly enhances generalization capabilities, particularly in noisy and imbalanced datasets. Notably, its user-friendly design facilitates integration with existing open-source code, enhancing computational efficiency and scalability.
Extensive experiments validate Robust-GBDT's superiority over other noise-robust methods, establishing a new standard for accurate classification amidst label noise. This research heralds a paradigm shift in machine learning, paving the way for a new era of robust and precise classification across diverse real-world applications.

\end{abstract}

\begin{keyword}
Label Noise Learning \sep Robust Loss\sep GBDT  \sep Imbalanced Learning
\end{keyword}

\end{frontmatter}

\section{Introduction}
\label{s:intro}

Tabular data classification is a fundamental task in machine learning, relying heavily on the accuracy of labels. However, labels often contain errors stemming from various sources such as human annotators, data collection procedures, ambiguous labeling rules, or data outliers. These inaccuracies introduce noise into the training data, hampering the ability of models to discern true patterns and achieve satisfactory performance \cite{frenay2013classification}.

To mitigate the challenges posed by label noise, researchers have turned to ensemble methods. Bagging \cite{breiman1996bagging} has proven effective in handling datasets with significant label noise  \cite{dietterich2000experimental}, leveraging the predictions of multiple models to yield a more robust outcome. However, its performance diminishes when dealing with datasets containing small amounts of noise. 

In contrast, boosting methods \cite{freund1997decision,friedman2001greedy} excel in achieving higher predictive accuracy and have demonstrated outstanding performance in low or noise-free scenarios. This success is attributed to iterative refinement of model predictions and a focus on misclassified instances. Yet, boosting is sensitive to substantial label noise, which can degrade its performance.

Recognizing boosting's potential, numerous researchers have sought to enhance its capabilities, making it more adept at handling substantial noise \cite{ wang2019splboost, li2018boosting, luo2023trboost}. However, these approaches have not been devoid of limitations. For instance, many are tailored exclusively for binary classification, making them unsuitable for more complex multi-class classification tasks. Moreover, they struggle to address issues associated with imbalanced datasets and missing values, common challenges in real-world applications. Additionally, these methods may not effectively harness parallel processing and GPU acceleration, hindering computational efficiency, especially when dealing with high-dimensional or large-scale datasets. As a result, their practical applicability has been significantly constrained.

Advanced Gradient Boosting Decision Trees (GBDT) models, exemplified by XGBoost \cite{chen2016xgboost} and LightGBM \cite{ke2017lightgbm}, unite state-of-the-art learners and optimization techniques. This fusion results in improved accuracy, faster training, and enhanced handling of intricate datasets. These models surmount the computational efficiency constraints that often hamper current noise-robust boosting methods. Consequently, they have emerged as the most efficient tools for tabular data classification \cite{borisov2022deep, grinsztajn2022tree}, establishing themselves as the preferred choice for practical problem-solving.
However, when employing the cross entropy loss, which is sensitive to label noise due to its nonsymmetric and unbounded nature \cite{zhang2018generalized}, the impact of noise is amplified.

At the core of label-noise learning lies the challenge of avoiding overfitting on noisy datasets, and the robustness of boosting algorithms is influenced by three key factors: the choice of loss function, the construction of base learners, and the implementation of regularization techniques.
Given the proven effectiveness of advanced GBDT models in utilizing efficient learners and effective regularization techniques, a coherent strategy to enhance the robustness of GBDT is to seamlessly integrate robust loss functions, tailored to mitigate the impact of noisy labels.
However, it is crucial to note that advanced GBDT models leverage Newton's method for optimization, which necessitates the loss function to be positive definite within a specific region. Consequently, some existing robust loss functions may not directly satisfy this requirement, necessitating careful consideration and adaptation.

In this paper, we present a groundbreaking theoretical insight: nonconvex loss functions, when subject to specific constraints, can be effectively employed in GBDT models. These nonconvex losses act as a form of regularization, mitigating overfitting issues caused by outliers. Building upon this theoretical foundation, we introduce Robust-GBDT, a novel robust gradient boosting model that seamlessly integrates the advanced GBDT framework with robust loss functions. Furthermore, we enhance the existing repertoire of robust losses by introducing the Robust Focal Loss, a tailored loss function designed to tackle the challenging problem of class imbalance.

Robust-GBDT outperforms existing bagging and robust boosting techniques, delivering more accurate predictions and significantly enhancing generalization capabilities, even in scenarios plagued by label noise or class imbalance. By harnessing the synergy between advanced GBDT and robust loss functions, our approach unlocks a new level of resilience and accuracy in classification tasks.

Moreover, Robust-GBDT offers user-friendly implementation and can seamlessly integrate with existing open-source code by simply replacing the objective function with RFL. This feature allows it to harness the rapid training advantages inherent in existing GBDT models, empowering it to tackle complex datasets effectively and yield numerous other benefits.

The main contributions are summarized as follows:

\begin{itemize}
    \item \textbf{Theoretical Insight}: The paper establishes that the loss function employed in advanced Gradient Boosting Decision Trees (GBDT), particularly those based on Newton's method, need not necessarily exhibit global convexity. Instead, the loss function only requires convexity within a specific region. This theoretical insight challenges the conventional wisdom and opens up the possibility of using nonconvex robust loss functions with GBDT.
    \item \textbf{Robust Focal Loss}: We enhance existing robust loss functions and introduce a novel robust loss function called Robust Focal Loss. This loss function is specifically designed to address class imbalance, a common issue in many real-world datasets, in addition to label noise.
    \item \textbf{Robust-GBDT Model}: We introduce a new noise-robust boosting model called Robust-GBDT. This model seamlessly integrates the advanced GBDT framework with robust loss functions, making it resilient to label noise in classification tasks.
    \item \textbf{Improved Generalization}: By leveraging the combination of advanced GBDT and robust loss functions, Robust-GBDT generates more accurate predictions and significantly enhances its generalization capabilities, particularly in scenarios marked by label noise and class imbalance.
    \item \textbf{User-Friendly Implementation}: Robust-GBDT is designed to be user-friendly and can easily integrate with existing open-source code. This feature allows for efficient handling of complex datasets while improving computational efficiency. We have made the code open source, which is available at \url{https://github.com/Luojiaqimath/Robust-GBDT}.
\end{itemize}

\section{Related Work}
\label{s:rel}
\subsection{Label-noise Learning}
Label-noise learning is an active topic in machine learning and pertains to the challenging task of training models to effectively handle and mitigate errors or inaccuracies present in the training labels, improving their robustness and generalization.
In the context of addressing label noise, methodologies can be roughly classified based on their incorporation of noise modeling, resulting in two distinct categories: sample-based methods and model-based methods.

Sample-based methods primarily aim to comprehensively characterize the inherent noise structure that serves as a guiding principle during the training phase. The key idea of these approaches is to extract noise-free information from the dataset, allowing the algorithm to obtain better performance and generalization. Methods commonly used include label correction \cite{song2019selfie,gong2022class, wang2021noise}, sample selection \cite{xia2021granular,xia2021random,sun2021co,xia2022adaptive} and weight assignment \cite{liu2015classification, wang2017multiclass,yao2020dual}.

On the contrary, model-based methods avoid explicitly modeling noise. Instead, they are designed to yield models resilient to the destabilizing effects of noise. These techniques function at their core by implementing adept regularization mechanisms. This strategic choice empowers the model to circumvent the risks of overfitting stemming from the presence of erroneous labels. Model-free approaches mainly include robust loss \cite{lyu2019curriculum, ma2020normalized, liu2020peer}, regularization \cite{srivastava2014dropout, szegedy2016rethinking, pereyra2017regularizing}, and ensemble learning \cite{wang2019splboost, li2018boosting, luo2023trboost}.

\subsection{Robust Boosting Algorithm for Tabular Data}
Friedman et al. \cite{friedman2000additive} introduced two robust boosting algorithms: LogitBoost and GentleBoost. LogitBoost enhances the logistic loss function through adaptive Newton steps, assigning reduced penalties to misclassified samples with large margins, often associated with noisy data. Consequently, it exhibits higher noise resilience compared to AdaBoost. GentleBoost optimizes the exponential loss function like AdaBoost, but its adaptive Newton step for calculating optimal base learners mimics LogitBoost. This new method in weak learners gives GentleBoost a more careful rise in sample weight, making it better at dealing with noise.
A similar approach can be observed in MadaBoost \cite{domingo2000madaboost}, which restricts example weights by their initial probabilities. This approach limits the increase of weights assigned to examples, which is observed in AdaBoost.

Theoretical findings \cite{li2018boosting, long2008random} highlight the noise robustness of boosting with nonconvex loss functions. Consequently, researchers have dedicated efforts to investigate the robustness of nonconvex loss-based boosting. An example is RobustBoost \cite{freund2009more}, derived from Freund's Boost-by-Majority algorithm \cite{freund1995boosting}, which demonstrates greater noise robustness than AdaBoost and LogitBoost. SavageBoost \cite{masnadi2008design} introduces the Savage loss function, trading convexity for boundedness. Additionally, Miao et al. proposed RBoost \cite{miao2015rboost}, incorporating a more robust loss function and numerically stable base learners for increased robustness. Emphasizing the significance of nonconvex losses due to their diminishing gradient attributes, Li et al. introduced ArchBoost \cite{li2018boosting}, a provably robust boosting algorithm. Within the ArchBoost framework, a family of nonconvex losses is introduced, resulting in Adaptive Robust Boosting.

In further attempts to enhance boosting's robustness, various new learners have been proposed. SPBL\cite{pi2016self} and SPLBoost\cite{wang2019splboost} both utilize self-paced learning to augment AdaBoost, albeit with distinct approaches. While SPBL adopts a max-margin perspective, SPLBoost adopts a statistical viewpoint. Luo et al.'s TRBoost \cite{luo2023trboost} is a novel GBDT model based on the trust-region algorithm, compatible with a wide array of loss functions, including robust options, as it doesn't necessitate positive definiteness of the Hessian matrix.

\subsection{Robust Loss}
The Categorical Cross Entropy (CCE) stands as the first choice for classification tasks due to its rapid convergence and strong generalization capabilities. However, when confronted with noisy labels, the robust Mean Absolute Error (MAE) loss, as demonstrated by \cite{ghosh2017robust}, exhibits superior generalization compared to the CCE loss. This superiority arises from the MAE loss being the sole candidate that meets the stipulated condition. Nonetheless, the MAE loss encounters limitations that MAE can lead to heightened training challenges and diminished performance.

Consequently, to harness the strengths of both MAE and CCE losses, the Generalized Cross Entropy (GCE) loss was introduced in \cite{zhang2018generalized}. This loss function represents a broader category of noise-robust losses, encapsulating the benefits of both MAE and CCE. 
Another innovative approach is the bi-tempered loss \cite{amid2019robust}, which provides an unbiased generalization of the CCE by leveraging the Bregman divergence.
Additionally, drawing inspiration from the symmetry of the Kullback-Leibler divergence, the Symmetric Cross Entropy (SCE) \cite{wang2019symmetric} was conceptualized. SCE combines the standard CCE loss with a noise tolerance term known as the reverse CCE loss. 

Meanwhile, the Curriculum Loss (CL) \cite{lyu2019curriculum}, serving as a surrogate for the 0-1 loss function, not only furnishes a tight upper bound but also extends seamlessly to multi-class classification scenarios.
The Active-passive Loss (APL) \cite{ma2020normalized} melds two distinct robust loss functions. An active loss component strives to maximize the probability of class membership, while a passive loss component endeavors to minimize the probability of belonging to other classes. Moreover, existing methods often require practitioners to specify noise rates, hence researchers propose a new loss function named Peer Loss (PL) \cite{liu2020peer}, which enables learning from noisy labels and does not require a specification of the noise rates. Recently, researchers have introduced asymmetric loss functions, a novel family tailored to meet the Bayes-optimal condition. These functions extend various commonly-used loss functions and establish the necessary and sufficient conditions to render them asymmetric, thus enhancing noise tolerance.\cite{zhou2023asymmetric}.

\section{Robust-GBDT Algorithm}
\label{s:method}
In current GBDT implementations, each decision tree produces a single output variable. However, when there are multiple outputs required, GBDT generates multiple trees, each corresponding to one of the output variables \cite{zhang2020gbdt}. 
In the case of multiclass classification, GBDT typically employs the "one-vs-all" approach. This means that for each class, a separate binary classification problem is solved.
Therefore, in this section, we discuss the results specifically tailored for binary classification.

\subsection{Second-order GBDT Algorithm}

Given a dataset 
$\mathcal{D}=\{(\mathbf{x}_1, y_1), (\mathbf{x}_2, y_2), \cdots, (\mathbf{x}_n, y_n)\}$~($\mathbf{x}_i \in \mathbb{R}^m, y_i\in \{0, 1\}$), where $\mathbf{x}_i$ is the $i$th sample and $y_i$ is the corresponding ground-truth label.
Let $l(y, p)$ denote the loss function for binary classification, where $p$ represents the probability for the class labeled as $y = 1$. This function satisfies the conditions $\lim_{p\to 0} l(0, p) = 0$ and $\lim_{p\to 1} l(1, p) = 0$.

In GBDT case, the objective of iteration $t+1$ ($t\ge 0$) is given as follows 
\begin{equation}
\label{e.objfun}
\begin{split}
    \mathcal{L}^{t+1} & = \sum_{i=1}^{n}l(y_i, p^{t+1}_i),\\
                    & = \sum_{i=1}^{n}l(y_i, S(z^{t+1}_i)),\\
                    & = \sum_{i=1}^{n}l(y_i, S(z^{t}_i+\alpha f_{t+1}(\mathbf{x}_i))),
\end{split}
\end{equation}
where $f_{t+1}$ is a new tree , $\alpha$ is the learning rate,
$z^{t}_i = z^{0}_i+\alpha \sum_{j=1}^{t}f_{j}(\mathbf{x}_i)$ is model's raw prediction, and $p^{t+1}_i = S(z^{t+1}_i) = 1/(1+e^{-z^{t+1}_i})$ is obtained by applying the Sigmoid function.
It's worth noting that $z^{0}_i$ is referred to as the initial prediction. Typically, we set $z^{0}_i$ to zero for all $i=1,2,\cdots,n$, as this selection has minimal impact on the final results.

For notational convenience, we define
\begin{equation}
\label{e.binary}
\hat{p} = 
\begin{cases}
p, & y=1,\\
1-p, & y=0, 
\end{cases}
\end{equation}
$\hat{p}$ is the probability for the ground-truth class. Therefore, we can rewrite $l(y, p)$ as $l(1, \hat{p})$ ($l(\hat{p})$ for simplicity) and have $\lim\limits_{\hat{p}\to 1} l(\hat{p})=0$,  $\forall y \in \{0, 1\}$.

For XGBoost and LightGBM, Newton's method is employed to optimize the regularized objective, which has the following formulation:
\begin{equation}
\label{e.objective}
    \mathcal{\widetilde{L}}^{t+1} = \sum\limits_{i=1}^{n}[g^{t}_i f_{t+1}(\mathbf{x}_i)+\frac{1}{2}h^{t}_i f_{t+1}(\mathbf{x}_i)^2] + \Omega(f_{t+1}),\\
\end{equation}
where $\Omega(f_{t+1})$ is a regularization parameter, $g^{t}_i$ is the gradient, and $h^{t}_i$ is the Hessian defined as follows
\begin{equation}
\label{e.binary_grad}
g^{t}_i = \frac{\partial l}{\partial z^{t}_i} = \frac{\partial l}{\partial \hat{p}^{t}_i}\frac{\partial \hat{p}^{t}_i}{\partial p^{t}_i}\frac{\partial p^{t}_i}{\partial z^{t}_i} = 
\frac{\partial l}{\partial \hat{p}^{t}_i}(2y_i-1)(\hat{p}^{t}_i(1-\hat{p}^{t}_i))
\end{equation}

\begin{equation}
\label{e.binary_hess}
h^{t}_i = \frac{\partial^2 l}{\partial (z^{t}_i)^2} = \frac{\partial^2 l}{\partial (\hat{p}^{t}_i)^2}(\hat{p}^{t}_i(1-\hat{p}^{t}_i))^2
+\frac{\partial l}{\partial \hat{p}^{t}_i}(\hat{p}^{t}_i(1-\hat{p}^{t}_i)(1-2\hat{p}^{t}_i)).
\end{equation}
Here we use the fact  that $\frac{\partial p^{t}_i}{\partial z^{t}_i} = p^{t}_i(1-p^{t}_i) = \hat{p}^{t}_i(1-\hat{p}^{t}_i)$.

For a fixed tree structure, $f_{t+1}(\mathbf{x})$ is a piecewise constant function and can be written as $f_{t+1}(\mathbf{x}) =\sum_{j=1}^{J}w_{j} I_{j}$, where $I_j$ is the instance set of leaf $j$. Let $\Omega(f_{t+1})=\frac{1}{2}\lambda\sum_{j=1}^{J}w_j^2$, $\lambda \ge 0$,  \eqref{e.objective} can be rewritten as follows:
\begin{equation}
\label{e.objective2}
\begin{split}
    \mathcal{\widetilde{L}}^{t+1} &= \sum_{i=1}^{n}[g^{t}_i f_{t+1}(\mathbf{x}_i)+\frac{1}{2}h^{t}_i f_{t+1}(\mathbf{x}_i)^2] + \Omega(f_{t+1}),\\
    & = \sum_{j=1}^{J}[(\sum_{i\in I_j} g^{t}_i)w_j + \frac{1}{2}( \sum_{i\in I_j} h^{t}_i + \lambda)w_j^2],
\end{split} 
\end{equation}

Since $f(\mathbf{x})$  is separable with respect to each leaf, we can calculate the optimal weight $w^{*}_j$ of leaf $j$ by
\begin{equation}
\label{e.optweight}
w^{*}_j = \frac{-\sum_{i\in I_j}g_{i}^{t}}{\sum_{i\in I_j}h_{i}^{t}+\lambda},
\end{equation}
and compute the corresponding optimal objective by
\begin{equation}
\label{e.optobj}
\mathcal{\widetilde{L}}^{*}_j = -\frac{1}{2}\frac{(\sum_{i\in I_j}g_{i}^{t})^2}{\sum_{i\in I_j}h_{i}^{t}+\lambda}.
\end{equation}

This score is like the impurity score, hence we employ the equation 
\begin{equation}
\label{e.split}
gain = \frac{1}{2}\Bigg[\frac{(\sum_{i\in I_{L}}g_{i})^2}{\sum_{i\in I_{L}}h_{i}+\lambda}+\frac{(\sum_{i\in I_{R}}g_{i})^2}{\sum_{i\in I_{R}}h_{i}+\lambda}-\frac{(\sum_{i\in I}g_{i})^2}{\sum_{i\in I}h_{i}+\lambda}\Bigg] \nonumber
\end{equation}
to determine whether a tree should be grown and to find the split feature and split value.
Here $I_{L}$ and $I_{R}$ are the instance sets of left and right nodes after the split, and $I = I_L \cup I_R$ is the sample set of the father node.

In practice, three rules control the leaf split and mitigate the risk of overfitting:
\begin{itemize}
    \item \textbf{Number of samples}: If the number of samples within a leaf is insufficient, then we stop the leaf split.
    \item \textbf{Sum of Hessian}: If the sum of hessian $\sum_{i\in I_{j}}h_{i}^{t}$ in a leaf is below a small threshold, specifically $\sum_{i\in I_{j}}h_{i}^{t}<\epsilon$, we stop the growth.
    \item \textbf{Loss reduction}: If the first two conditions are met, but $gain$ falls below a threshold $\delta$, we also stop the growth process to prevent overfitting. 
\end{itemize}


\subsection{Convexity Analysis}

From Eq.~\eqref{e.optweight}  and the split rules, it can be seen that we don't need all the Hessian a leaf to be positive. Instead, we only require $\sum_{i\in I} h_{i}\ge \epsilon$, which tells us that certain $h_i$ can be non-positive and offers the potential to employ nonconvex loss.

In this subsection, we try to answer the following two questions:
\begin{enumerate}
    \item \textbf{What are the minimum requirements for the convexity of the Hessian in the model, that is, the necessary condition?}
    \item \textbf{What is the impact of the non-positive Hessian on the model?}
\end{enumerate}

\subsubsection{\textbf{Necessary condition}}

As per Eq.~\eqref{e.binary_hess}, it's apparent that Hessian $h$ is bound by the probability $\hat{p}$. Consequently, we define a function $h$ in terms of $\hat{p}$, represented as $H(\hat{p})$. 
These two observations lead to the following theorems, providing a quantitative description.

\begin{theorem}[Convexity for Binary Classification \uppercase\expandafter{\romannumeral1}]
\label{t.prob}
In the case of binary classification, the necessary condition to enable a leaf to split is that the Hessian $h$ should be positive when  $\hat{p} \in [0.5,1)$.
\end{theorem}

\begin{proof}
Consider the simplest case with only one instance, the iteration of GBDT tells us we need to learn a sequence of $\{\hat{p}^t\}_{t=0}^{+\infty}$ such that $\lim\limits_{t \to +\infty} \hat{p}^t=1$. Due to the fact that every convergent sequence has a monotonic subsequence, we can assume that the sequence $\{\hat{p}^t\}_{t=0}^{+\infty}$ is monotonic in this proof.

To create a new tree, it is crucial that all values of $h^t$ remain positive throughout the entire iteration, meaning $h^t > 0$ for all $t$.
It is noteworthy that for every $\hat{p}$ in $[\hat{p}^0, 1)$, it is possible to choose a proper learning rate $\alpha$ such that $\hat{p} = S(z^{0}+\alpha f_{1})$. Consequently, $\forall \hat{p}^t \in [\hat{p}^0, 1)$, we can affirm that $H(\hat{p}^t)>0$.
Since $\hat{p}^0=S(z^{0}=0)=0.5$, we establish the convexity of Hessian.

Since $\{\hat{p}^t\}_{t=0}^{+\infty}$ may not consistently exhibit monotonic behavior in practical scenarios, this condition is a necessary condition. 
\end{proof}

Given that $p = S(z)$ is the estimated probability for label $y=1$, based on Theorem \ref{t.prob} and Eq.~\eqref{e.binary}, we have the following theoretical result:
\begin{theorem}[Convexity for Binary Classification \uppercase\expandafter{\romannumeral2}]
\label{t.value}
For a sample with label $y=1$, the Hessian $h$ should have positive values when learners' prediction $z \geq 0$. On the other hand, when the label is $y=0$,  $h$ should be positive when $z \leq 0$.
\end{theorem}

\begin{proof}
For a sample with label $y=1$, $\hat{p} \in [0.5, 1)$ means $p \in [0.5,1)$ and $z\ge 0$. 
On the other hand, when the label is $y=0$,  $\hat{p} \in [0.5, 1)$ means $p \in (0, 0.5]$, resulting in $z\le 0$.
\end{proof}

In situations involving multiple output variables, the current GBDT generates separate binary trees, with each tree corresponding to one of the output variables. Consequently, the convexity of the Hessian in the context of multi-class classification can be deduced based on the previously established theorems.
Assuming the GBDT output is represented as $\mathbf{z} = (z_1, z_2, \dots, z_C)$, where $z_k$ denotes the output of the $k$th tree and $C$ represents the number of classes, the theoretical result for multi-class classification is as follows:

\begin{theorem}[Convexity for Multi-class Classification]
\label{t.mvalue}
For a sample with label $y=k$, the Hessian $h$ should have positive values when learners' prediction $z_k \geq 0$. On the other hand, $h$ should be positive when $z_j \leq 0, j\neq k$. 
\end{theorem}

\subsubsection{\textbf{Impact of non-positive Hessian}}
Based on the preceding discussion, we can conclude that the sign of the Hessian acts as a regularizer through the following two key aspects: 
\begin{itemize}
    \item \textbf{Sum of Hessian}: Some non-positive Hessian values make it easier for $\sum_{i\in I}h_i$ to become less than $\epsilon$, causing the leaf splitting process to terminate early.
    \item \textbf{Loss reduction}: The reduction in $\sum_{i\in I}h_i$ may lead $gain$ below a threshold $\delta$, which in turn can prohibit the growth of the tree.
\end{itemize}

The first one is obvious. For the second one, since the $gain$ is also affected by the sum of the gradient $\sum_{i\in I}g_{i}$ and $\lambda$, it is hard to say that non-positive Hessian must reduce the $gain$. Instead, we give some examples, which are in Appendix \ref{a.existence}, to show the existence of non-positive Hessian can indeed reduce the loss.

\subsection{Robust-GBDT Model}
We first introduce the definition of Robust-GBDT based on the above discussions:
\begin{definition}[Robust-GBDT]
We refer to the GBDT model based on Newton's method and equipped with a robust loss function that meets the necessary condition as \textbf{Robust-GBDT}.
\end{definition}

While there are several robust loss functions available, not all of them are a good fit for GBDT. That is because GBDT doesn't offer the same level of flexibility as deep learning when it comes to designed loss functions. Here we list the convexity of a few robust loss functions that are commonly used with the GBDT model:

\begin{itemize}
    \item Mean Absolute Error \cite{ghosh2017robust}: $l_{MAE}= 1-\hat{p}$
    \item Generalized Cross Entropy \cite{zhang2018generalized}: $l_{GCE} = \frac{1-\hat{p}^q}{q}, ~q\in(0,1]$
    \item Symmetric Cross Entropy \cite{wang2019symmetric}: $l_{SCE} = -\alpha log(\hat{p})+\beta(1-\hat{p})$
    \item Normalized Cross Entropy \cite{ma2020normalized}: $l_{NCE} = \frac{log(\hat{p})}{log(\hat{p})+log(1-\hat{p})}$
\end{itemize}

Fig.~\ref{f.robust_loss}(a) illustrates the graphical representation of different loss functions' Hessian values. It is discernible from the plot that GCE and SCE conform to the conditions outlined in Theorem \ref{t.prob}. Notably, SCE exhibits partial robustness to noise because only the MAE term is robust to noise.
In contrast, both MAE and NCE exhibit central symmetry. Consequently, their Hessians attain a value of 0 at $\hat{p}=0.5$, rendering them non-compliant with the specified condition.

\begin{figure*}[!t]
\centering
\subfigure[\small Hessian $H(\hat{p})$ of different losses]{\includegraphics[width=2.5in]{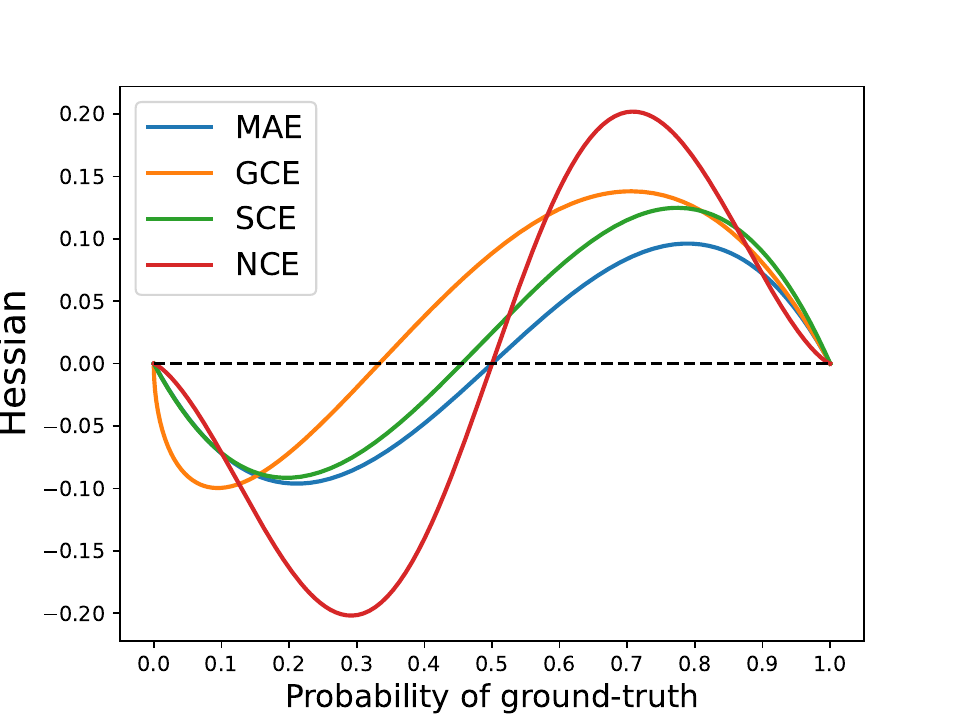}}
\subfigure[\small Hessian $H(\hat{p})$ of RFL]{\includegraphics[width=2.5in]{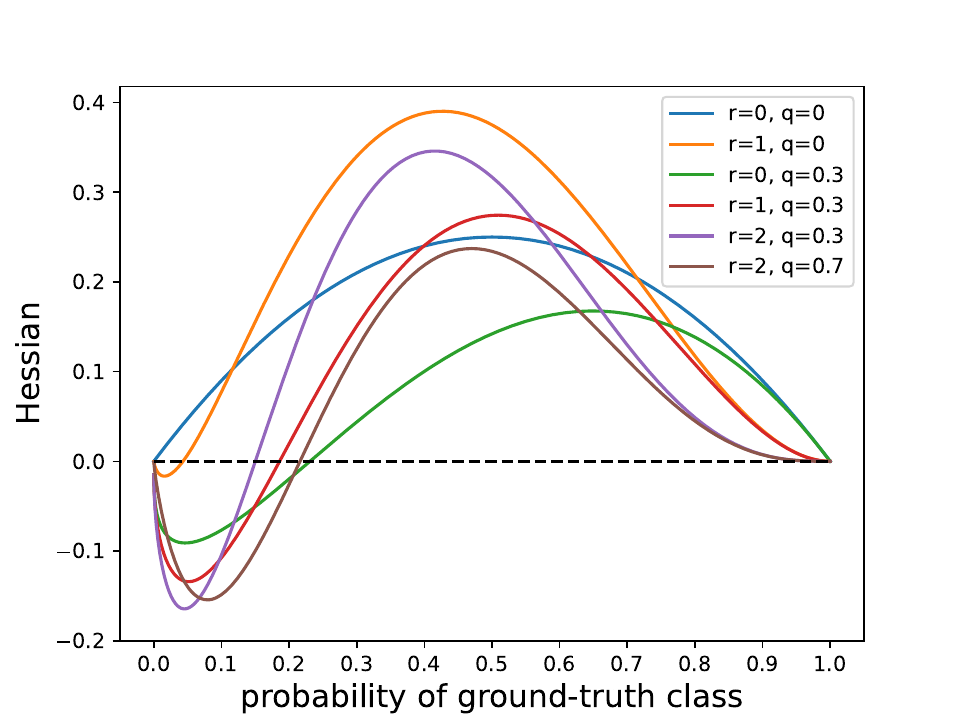}}
\caption{Different Robust Losses and their corresponding Hessian.}
\label{f.robust_loss}
\end{figure*}

Although GCE satisfies all the conditions, it does not account for class imbalance. Therefore, we introduce a novel loss function, termed the \textbf{Robust Focal Loss (RFL)}.
The RFL is defined as follows:
\begin{equation}
\label{e.rfl}
l_{RFL} = (1-\hat{p})^r\frac{1-\hat{p}^q}{q},
\end{equation}
where $r$ and $q$ are two hyperparameters, $r\ge 0$, $q\in (0,1)$. The loss is a combination of GCE and Focal Loss (FL) \cite{lin2017focal} $l_{FL} = -(1-\hat{p})^{r}\log{\hat{p}}$, see Fig.~\ref{f.robust_loss} (b).
When $r=0$, RFL degenerates into GCE; when $r>0$, $q \to 0$, RFL transforms into FL; and when $r=0$, $q\to 0$,  RFL is equal to CCE.

\begin{remark}
In order to enhance the adaptability of MAE and NCE to GBDT, we can introduce perturbations to $\hat{p}$  when $\hat{p} \le 0.5$.  For instance, one approach involves substituting $\hat{p}$ with $\hat{p} + \eta$, where $\eta$ is a small positive scalar. 
To make SCE global robustness, a clipping operation can be employed on $\hat{p}$, replacing it by $\max(\hat{p}, \eta)$.
Additionally,  we can add the term $(1-\hat{p})^r$ to make these losses account for class imbalance.
\end{remark}

\section{Experiments}
\label{s:exps}
\subsection{Experiments Setup}

\subsubsection{Datasets} 
We conducted comprehensive comparisons on a dataset comprising 25 datasets, specifically comprising 15 binary and 10 multi-class datasets.
The binary datasets are from the imbalanced-learn \footnote{\url{https://imbalanced-learn.org/}}
and the OpenML\footnote{\url{https://www.openml.org/}} repository, 
and the multi-class datasets come from the KEEL\footnote{\url{https://sci2s.ugr.es/keel/datasets.php}}
and the OpenML repository.
To assess the robustness of different algorithms, we introduced label noise by randomly selecting a certain proportion of the training data and flipping their labels. We varied the noise levels from 0\% to 40\% in increments of 10\%.
It's essential to note that we did not apply any specialized preprocessing techniques, including missing value imputation, to the datasets.
For more details, we refer the reader to Appendix \ref{a.data}.

\subsubsection{Benchmark models}
We introduce two models, Robust-XGB (RXGB) and Robust-LGB (RLGB), which are built upon XGBoost and LightGBM, respectively. We then conduct a comparative analysis of these models against six commonly employed ensemble baselines, namely AdaBoost (AB) \cite{freund1997decision}, LogitBoost (LB) \cite{friedman2000additive}, Bagging (BG) \cite{breiman1996bagging}, RandomForest (RF) \cite{breiman2001random},
XGBoost (XGB) \cite{chen2016xgboost}, LightGBM (LGB) \cite{ke2017lightgbm}. All these models can be applied to both binary and multi-class classification tasks.

\subsubsection{Evaluation metrics} 
For the binary classification tasks, we utilize \textbf{Area Under the Precision-Recall Curve (AUCPR)} as the evaluation metric, and for multi-classification problems, we chose \textbf{Accuracy} as the metric.

\subsubsection{Training procedure} 
For all datasets, we randomly hold out 80$\%$ of the instances as the training set and the rest as the test set.
We introduce different levels of noise, ranging from 0\% to 40\%, into the training data and then optimize the parameters using this noisy data.
Once we have determined the optimal parameters, we make predictions on the test set.
The entire process is repeated 5 times and the average evaluation score is reported.

\subsection{Performance Comparison}

In this subsection, we examine the robustness of various ensemble methods in the context of binary classification (15 datasets) and multi-class classification (10 datasets).

\subsubsection{Binary Classification}
Fig.~\ref{f.binary} illustrates the testing AUCPR as it varies with different levels of noise. The figure unmistakably highlights that RandomForest, Bagging, AdaBoost, and Logboost consistently underperform when compared to the four GBDT-based models. This performance gap becomes particularly pronounced when the dataset size is large, as evidenced by the PR and SA datasets. Furthermore, it's noteworthy that some of these models fail to yield results when faced with datasets containing missing values (dataset CR) or high-dimensional data (dataset SA).
Additionally, we observe that RFL, incorporated into imbalanced learning and label noise learning, enhances the robustness of Robust-XGB and Robust-LGB, resulting in superior performance compared to their non-robust counterparts, XGBoost and LightGBM. Notably, these two robust methods exhibit higher AUCPR scores even at noise levels as low as 0\%, and this advantage persists up to a noise level of 40\%.

To provide a comprehensive assessment, we have ranked the eight methods on a scale of 1 to 8, with 1 indicating the highest AUCPR score and 8 representing the lowest. Our evaluation encompasses a total of 75 training datasets derived from 15 diverse datasets and across five noise levels.
We calculate the sum of the datasets in which a particular algorithm is ranked within the top-\textit{n}. This information is visually represented in Fig.~\ref{f.rank_binary}. 
Here, the top-\textit{n} means that for a given algorithm, there are y-axis value datasets in which the algorithm achieves top-\textit{n} performance.
Additionally, we calculate the average rank based on 5 tests for each dataset-model combination, and the results are provided in Table~\ref{T.binary_rank}.
It is evident from the figure that Robust-GBDT significantly outperforms other competing models. This clear superiority confirms its enhanced resistance to label noise and class imbalance.

\begin{figure*}[!t]
\centering
\includegraphics[scale=0.35]{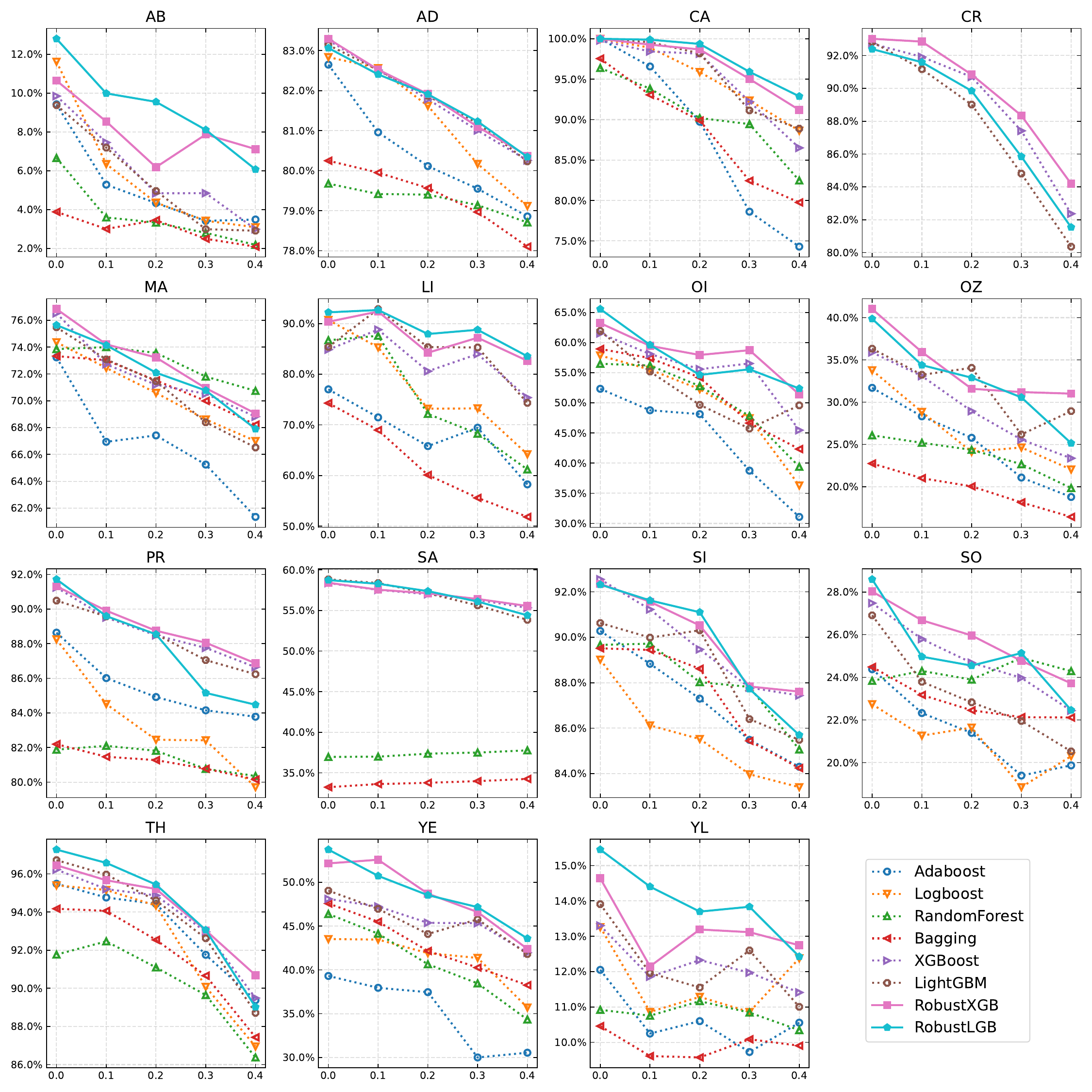}
\caption{Test AUCPR curves that change with different noise levels for the 8 ensemble methods on 15 binary classification datasets. Within each subfigure, the x-axis represents the noise level, while the y-axis corresponds to the AUCPR.}
\label{f.binary}
\end{figure*}

\begin{figure*}[!t]
\centering
\includegraphics[scale=0.35]{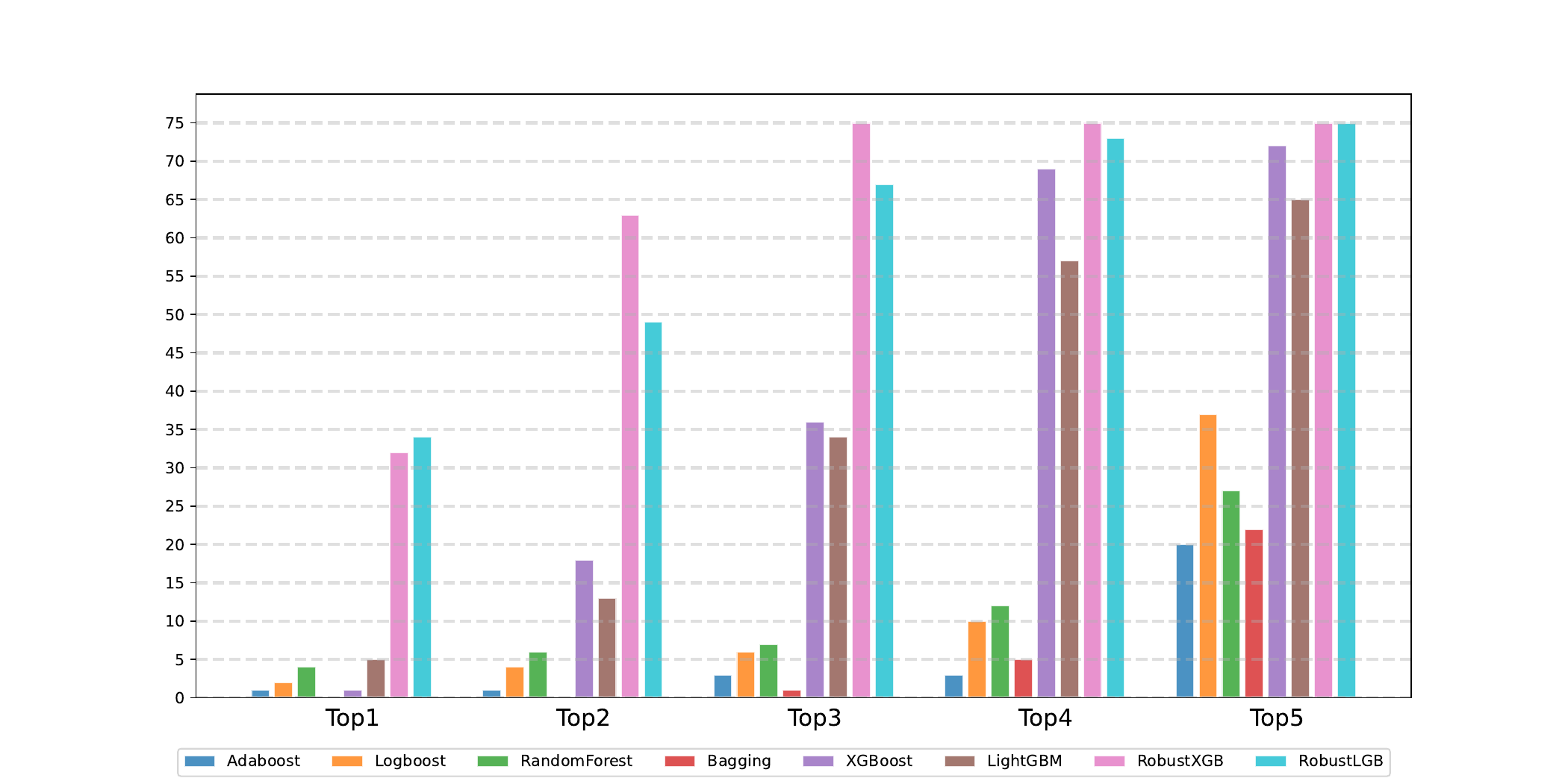}
\caption{Statistics of top-$n$ rank values on 75 datasets derived from 15 binary datasets with five noise levels. The y-axis represents the dataset number.}
\label{f.rank_binary}
\end{figure*}

\begin{table}[!ht]
\normalsize
\renewcommand\arraystretch{1.1}
\centering
\label{T.binary_rank}
\begin{adjustbox}{width=0.9\textwidth}
\begin{tabular}{lcccccc|cc}
    \toprule[2pt]
        Dataset & AB & LB & RF & BG  & XGB & LGB &  RXGB & RLGB\\
    \midrule[1.5pt]
        AB & 5.0 & 4.0 &  7.2 &   7.8 &  3.8 &  5.0 &  2.0 & \textbf{1.2} \\
    \hline
        AD & 6.0 & 4.2 & 7.6 & 7.4 & 3.2 & 3.0 & \textbf{1.6} & 2.8\\
    \hline
        CA & 6.2 & 3.4 & 5.8 & 6.4 & 4.0 & 2.8 & 2.0 &  \textbf{1.0}\\
    \hline
        CR & -  & - & - & - & 2.2 & 3.6 &  \textbf{1.0} &3.2\\
    \hline
        MA & 8.0 & 6.2 & 2.4 & 5.0 & 4.0 & 5.4 & \textbf{1.6} & 3.2\\
    \hline
        LI & 6.8 & 4.6 & 5.6 & 8.0 & 4.2 & 3.0 & 2.6 & \textbf{1.2}  \\
    \hline
        OI & 8.0 & 6.0 & 5.4 & 4.8 & 3.0 & 5.4 & \textbf{1.2} & 1.8 \\
    \hline
        OZ & 6.2 & 5.4 & 6.4 & 8.0 & 4.0 & 2.2 & \textbf{1.4} & 2.2\\
    \hline
        PR & 5.0 & 6.4 & 7.2 & 7.4 & 3.0 & 3.0 & \textbf{1.2} & 2.8\\
    \hline
        SA & -  & - &5.0 & 6.0 & 3.2 & 2.4 & \textbf{2.0} & 2.2 \\
    \hline
        SI & 6.2 & 8.0 & 4.8 &6.4 &  2.6 & 4.0 & \textbf{1.6} & 2.4\\
    \hline
        SO & 7.2 & 7.6 &  3.6 &  5.4 &  3.0 & 5.2 & \textbf{1.8} & 2.2\\
    \hline
        TH & 4.8 & 6.2 & 8.0 & 6.6 & 3.2 & 3.4 & 2.2 & \textbf{1.6} \\
    \hline
        YE & 8.0 & 6.2 & 6.6 & 5.2 & 3.4 & 3.6 & 1.6 & \textbf{1.4} \\
    \hline
        YL & 6.8 & 4.6 & 6.4 &  7.8 & 3.8 & 3.6 & 1.8 & \textbf{1.2} \\
    \midrule[1.5pt]
    Average & 6.41 & 5.65 & 5.80 & 6.48 & 3.38 & 3.74 & \textbf{1.73} & 2.02\\

    \bottomrule[2pt]
\end{tabular}
\end{adjustbox}
\caption{Average rank for different models on binary classification datasets. The best results for each dataset are in \large\textbf{bold}.}
\end{table}

\subsubsection{Multi-class Classification}
Fig.~\ref{f.multi} plots the testing accuracy that changes with different levels of noise.

Similarly, due to the high dimensionality issue, it is also infeasible to evaluate Adaboost and Logboost on dataset CV. 
While the advantages of Robust-XGB and Robust-LGB in multi-class classification, as depicted in Fig.~\ref{f.multi}, may not be as pronounced as those observed in binary classification results, the true strength of these methods becomes evident in Fig.~\ref{f.rank_multi} and Table~\ref{T.multi_rank}.
In this context, it becomes apparent that the superior performance of Robust-XGB and Robust-LGB observed in binary classification tasks continues seamlessly into multi-class classification.

\begin{figure*}[!t]
\centering
\includegraphics[scale=0.35]{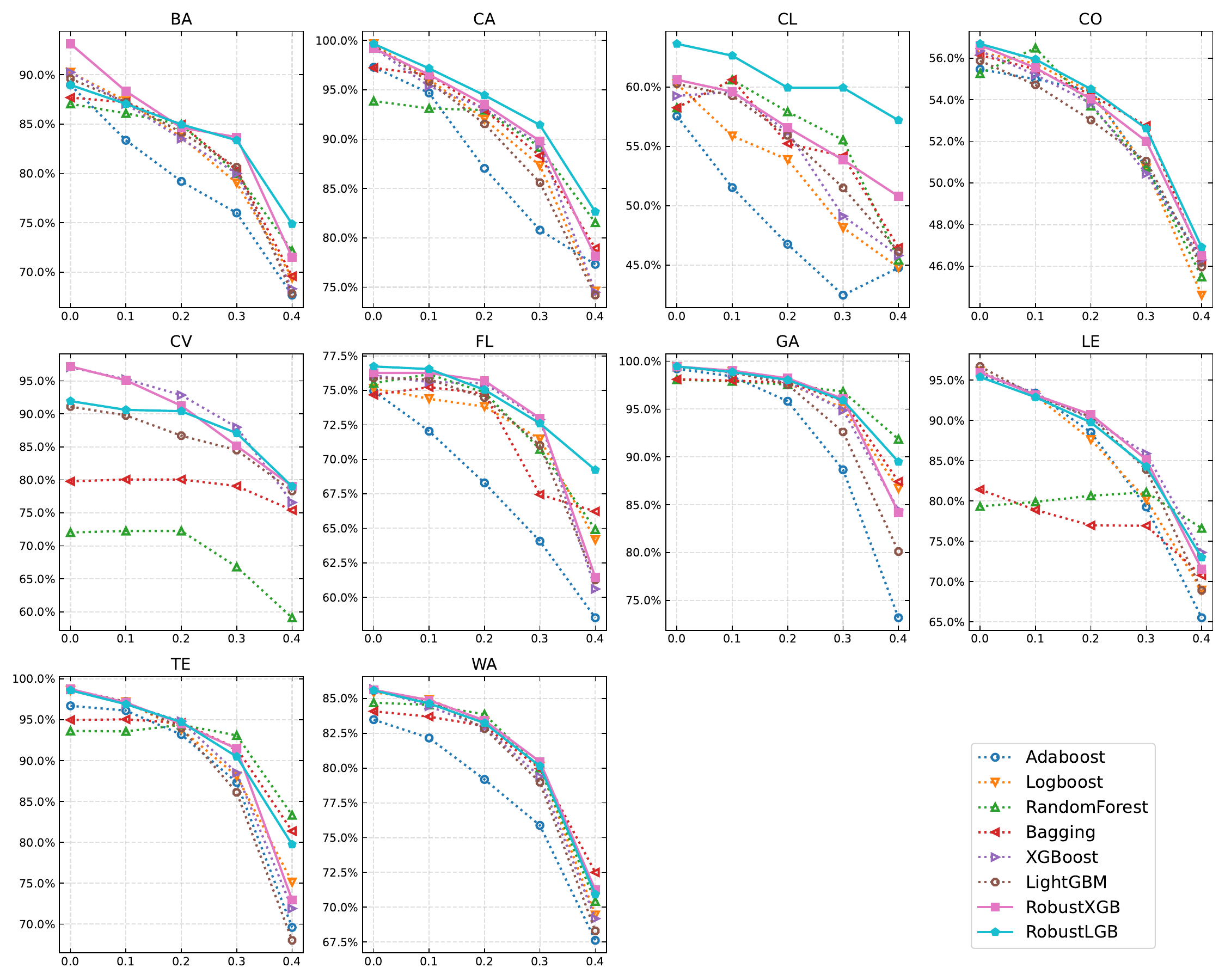}
\caption{Test accuracy scores that change with different noise levels for the 8 ensemble methods on 10 multi-class classification datasets. Within each subfigure, the x-axis represents the noise level, while the y-axis corresponds to the Accuracy.}
\label{f.multi}
\end{figure*}

\begin{figure*}[!t]
\centering
\includegraphics[scale=0.35]{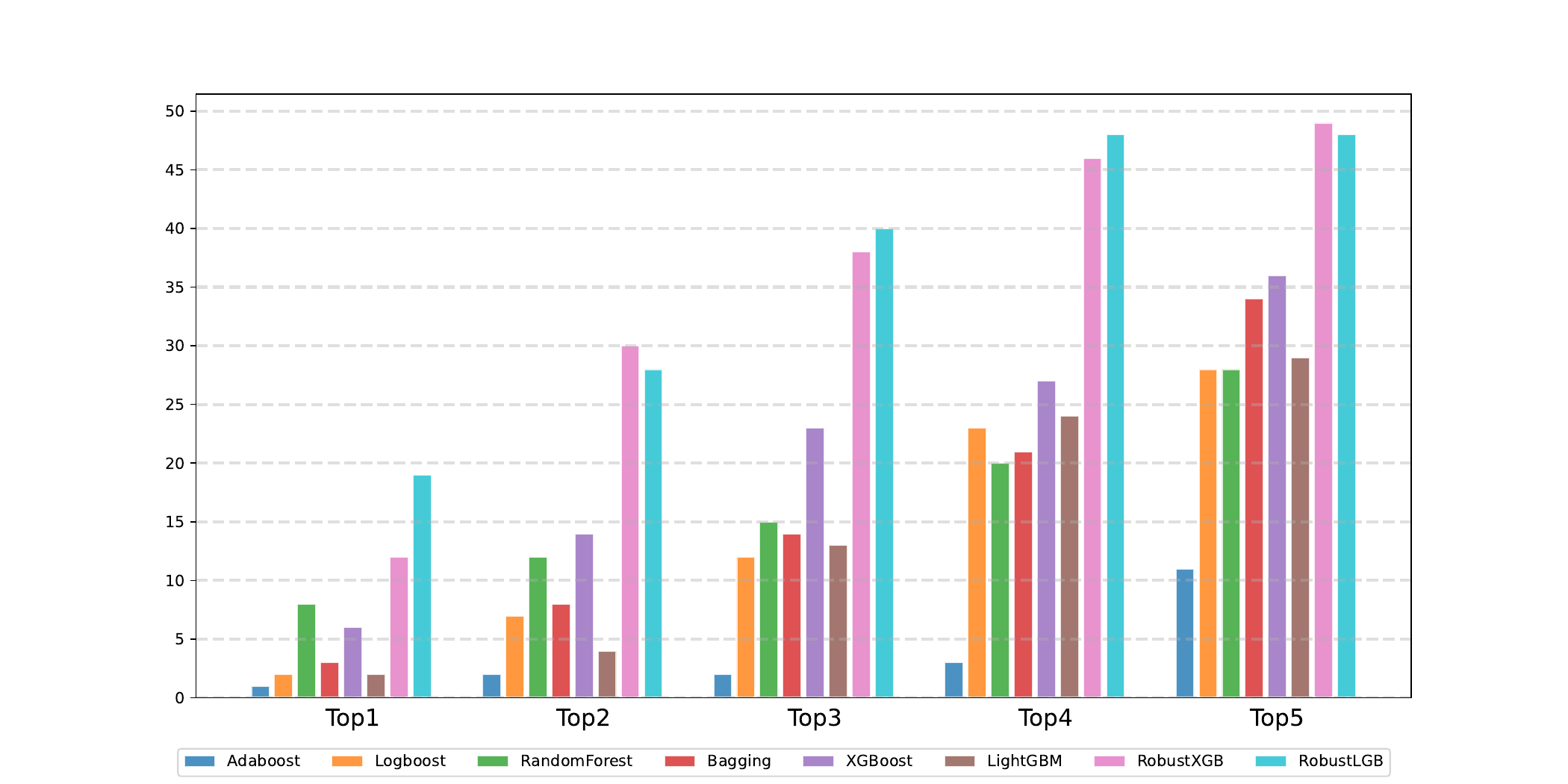}
\caption{Statistics of top-$n$ rank values on 50 datasets derived from 10 multi-class classification datasets with five noise levels. The y-axis represents the dataset number.}
\label{f.rank_multi}
\end{figure*}

\begin{table}[!ht]
\normalsize
\renewcommand\arraystretch{1.1}
\centering
\label{T.multi_rank}
\begin{adjustbox}{width=0.95\textwidth}
\begin{tabular}{lcccccc|cc}
    \toprule[2pt]
        Dataset & AB & LB & RF & BG  & XGB & LGB &  RXGB & RLGB\\
    \midrule[1.5pt]
        BA & 6.2 & 3.8 & 3.8 & 3.4 & 4.4 & 3.8 &  \textbf{1.6} & 2.4\\
    \hline
        CA & 6.4 & 4.4 & 5.0 & 4.0 &  4.4 & 5.6 & 3.0 & \textbf{1.0}\\
    \hline
        CL & 8.0 & 6.4 &  4.0 &  4.0 & 5.0 &  4.6 & 3.0 & \textbf{1.0} \\
    \hline
        CO & 5.4 &  5.0 & 5.6 &  3.8 & 5.2 & 6.4 & 3.2 &  \textbf{1.4}\\
    \hline
        CV & -  & - & 6.0 & 5.0 & \textbf{1.8} & 3.8 & \textbf{1.8} & 2.6\\
    \hline
        FL & 7.8 &  5.6 & 4.2 & 5.6 &  3.8 &  5.0 & 2.2 & \textbf{1.8}\\
    \hline
        GA & 7.0 &  3.8 & 4.8 &  4.8 &  4.2 & 5.6 & 2.4 & \textbf{2.2} \\
    \hline
        LE & 5.2 &  4.8 &  5.6 & 7.2 &  \textbf{2.2} & 3.8 & 2.6 &  4.4\\
    \hline
        TE & 6.4 & 4.6 & 4.2 & 3.8 & 3.6 & 5.2 & \textbf{3.2} &  3.4\\
    \hline
        WA & 7.8 &  3.2 &  3.8 & 4.8 & 5.0 & 5.4 & \textbf{1.8} & 3.2\\
    \midrule[1.5pt]
    Average & 4.48 & 3.24 & 3.13 & 3.09 & 2.64 & 3.28 & 1.67 & \textbf{1.56}\\
    \bottomrule[2pt]
\end{tabular}
\end{adjustbox}
\caption{Average rank for different models on multi-class classification datasets. The best results for each dataset are in \large\textbf{bold}.}
\end{table}

Based on these statistical findings, it becomes evident that Robust-XGB and Robust-LGB consistently outperform other methods, particularly outshining RandomForest, Bagging, AdaBoost, and Logboost. Although XGBoost, LightGBM, Robust-XGB, and Robust-LGB exhibit relatively similar performance in some datasets, an overall assessment reveals the superior performance of Robust-XGB and Robust-LGB.

\subsection{Ablation studies}
We further conduct some ablation experiments to investigate the impact of removing two key factors on the performance of our model. The two factors of interest are:
\begin{itemize}
    \item Factor $(1-\hat{p})^r$: We want to analyze how removing this factor affects model performance.
    \item Parameter $q$: We want to understand the contribution of value $q$ to the model's performance.
\end{itemize}
The experiments are conducted on three binary and three multi-class classification datasets, selected based on their imbalance ratios.

To investigate the influence of the factor $(1-\hat{p})^r$ on the model, we let $r=0$, degenerating RFL into GCE, and retrain the model. To assess the contribution of variable $q$ to the model's performance, we assign 0 to $q$, thereby converting RFL into FL
The results of these ablation experiments are displayed in Fig.~\ref{f.ablation}.

Our observations reveal that the removal of the factor $(1-\hat{p})^r$ from the model results in a small decline in performance when the dataset exhibits a low imbalance ratio. However, this decline becomes much more pronounced when the dataset exhibits a high imbalance ratio. These findings suggest that the factor $(1-\hat{p})^r$ indeed plays a role in enhancing the predictive power of robust loss, particularly in scenarios where the dataset faces pronounced imbalances.

In the ablation model where variable $q$ is altered, an obvious decrease in the model's performance is obtained. Notably, when the dataset is imbalanced, altering variable $q$ has a substantial adverse effect on the model's performance, indicating noise-robust loss can indeed improve the robustness of FL in the model's decision-making process.

\begin{figure}[ht]
\centering
\subfigure[\scriptsize ratio: 383/307]{\includegraphics[width=0.3\textwidth]{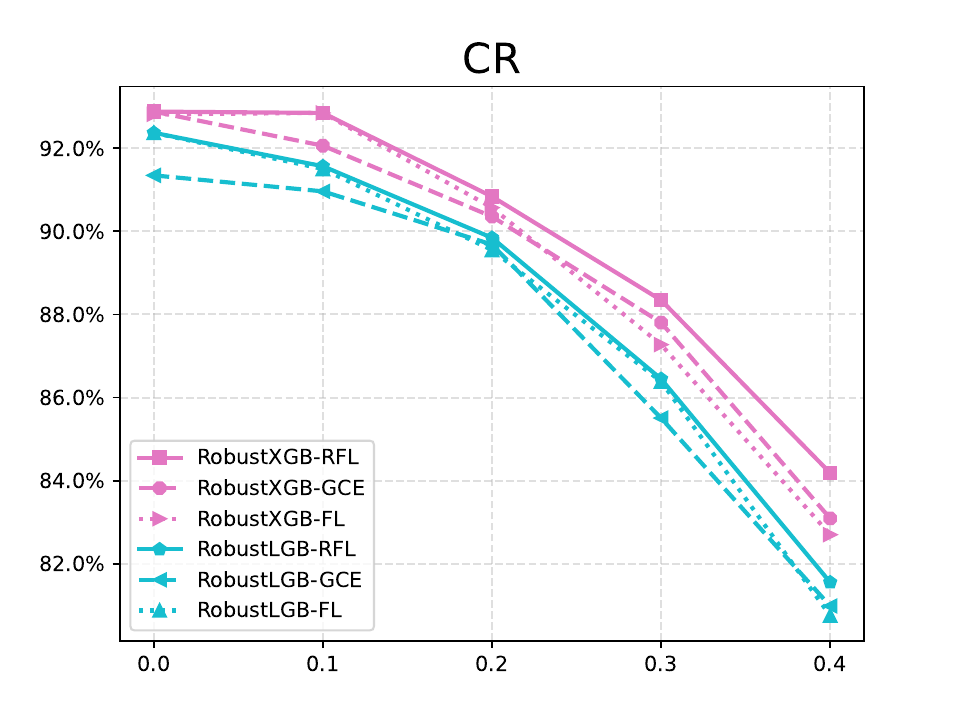}}
\subfigure[\scriptsize ratio: 896/41]{\includegraphics[width=0.3\textwidth]{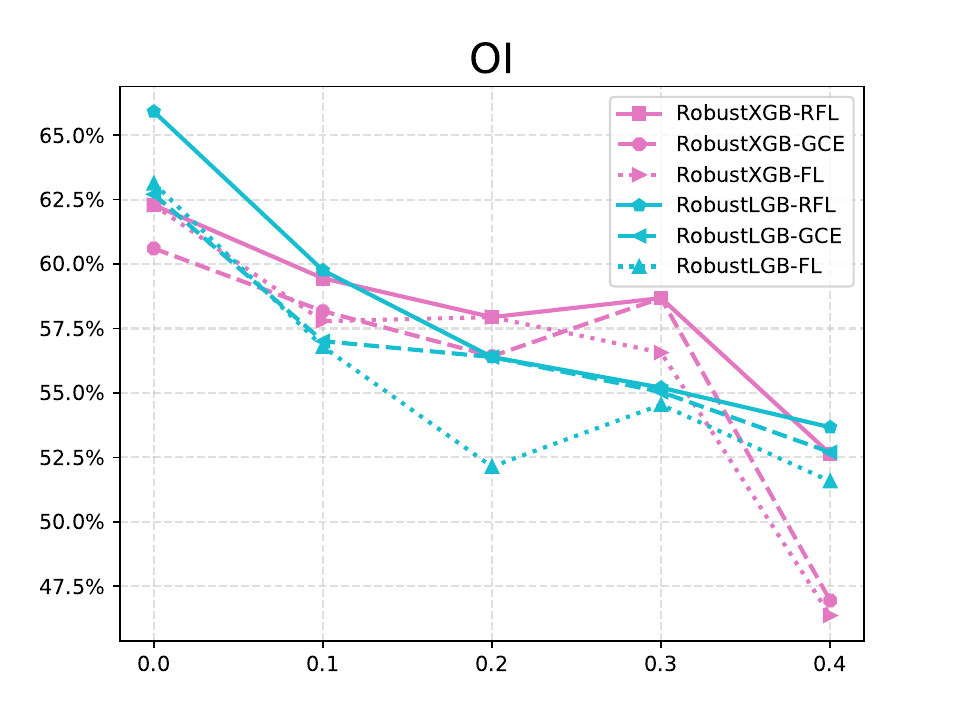}}
\subfigure[\scriptsize ratio: 4145/32]{\includegraphics[width=0.3\textwidth]{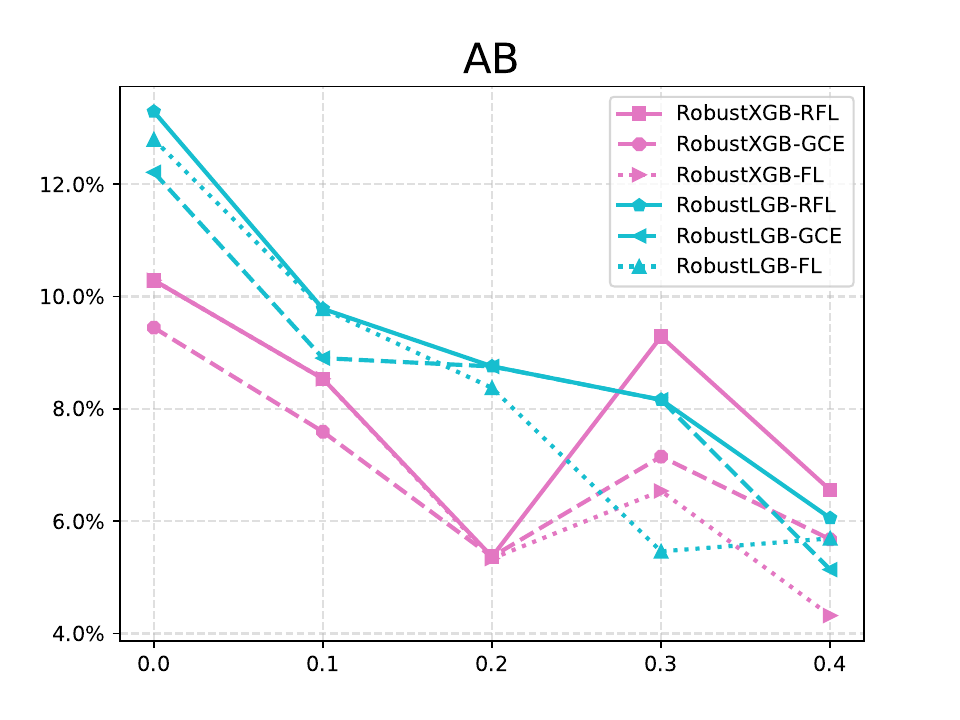}}\hfil
\subfigure[\scriptsize ratio:1692/1653/1655]{\includegraphics[width=0.3\textwidth]{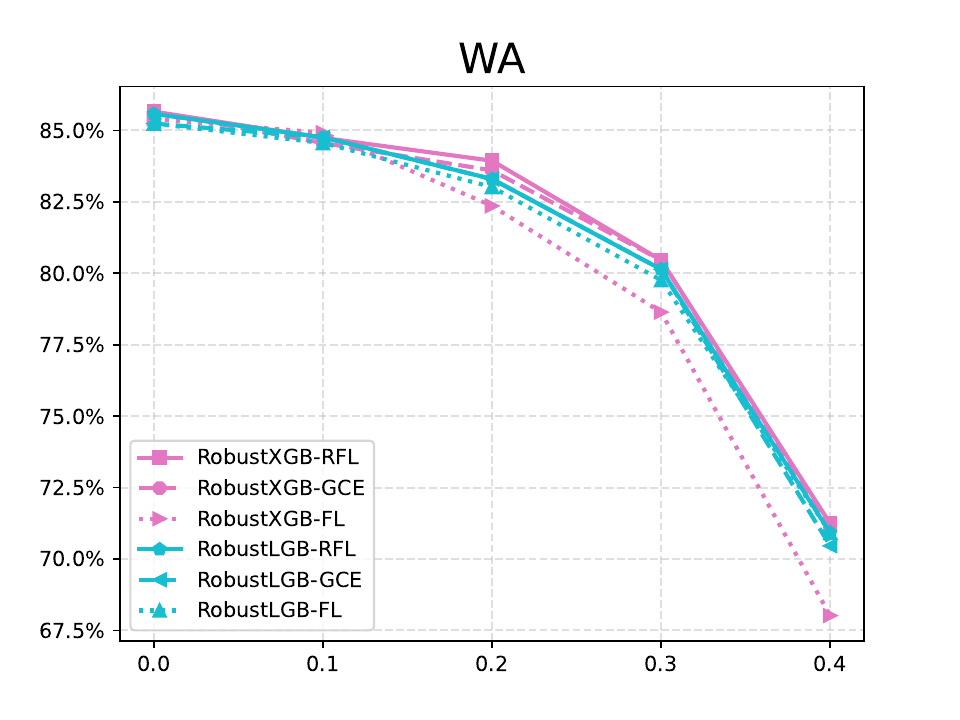}}
\subfigure[\scriptsize ratio: 147/211/239/95/43/331]{\includegraphics[width=0.3\textwidth]{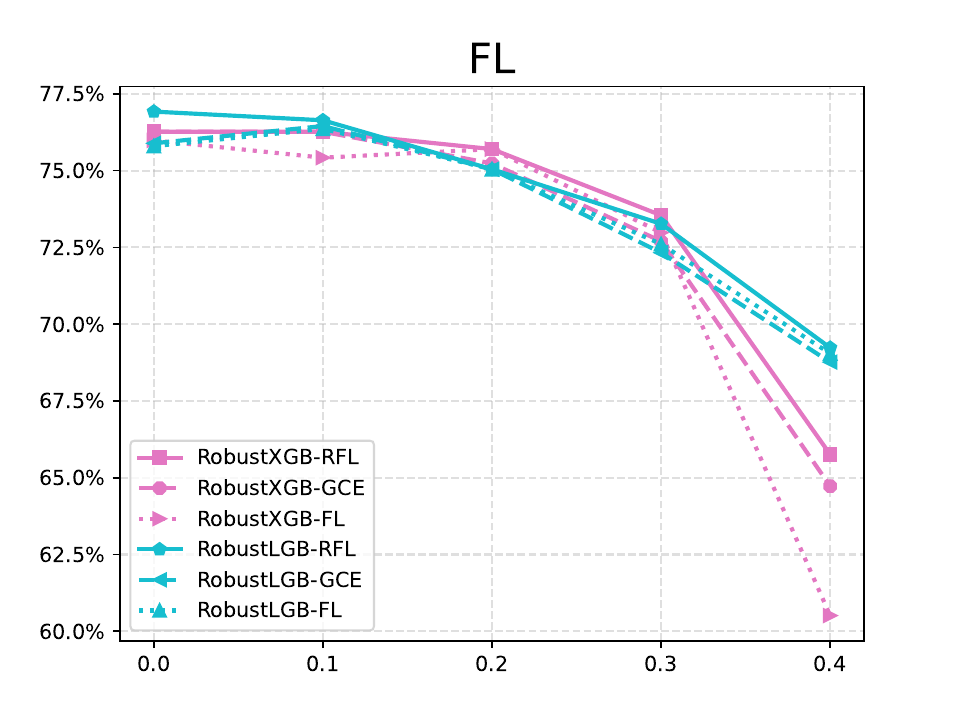}}
\subfigure[\scriptsize ratio: 160/54/35/35/13]{\includegraphics[width=0.3\textwidth]{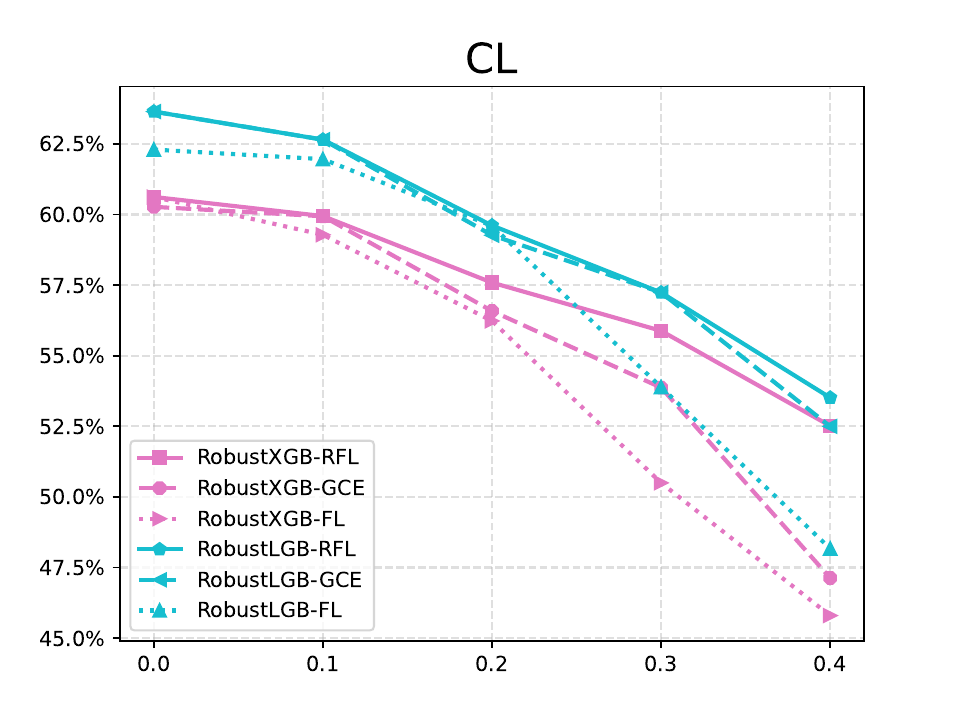}}
\caption{Results of ablation study. The first row is the 3 results of binary classification, together with the imbalanced ratio. The second row consists of 3 results for multi-class classification, along with the corresponding class distribution. Within each subfigure, the x-axis represents the noise level, while the y-axis corresponds to the AUCPR.}
\label{f.ablation}
\end{figure}

\section{Conclusions}
\label{s:conclu}
Boosting is a powerful ensemble method widely employed for classification tasks. However, its performance can be significantly impacted by label noise. To address this limitation, robust boosting algorithms have emerged as viable alternatives. However, most existing robust boosting methods are tailored for binary classification tasks and struggle with imbalanced data, missing values, and computational efficiency issues.

This paper establishes a robust theoretical foundation by demonstrating the adaptability of nonconvex loss functions, under specific constraints, for Newton's method-based Gradient Boosting Decision Trees (GBDT). This theoretical development broadens the scope of loss functions that GBDT models can effectively utilize. Expanding on this theoretical insight, we present a novel noise-robust boosting model known as Robust-GBDT. This model incorporates the GBDT framework with robust loss functions. Furthermore, we enhance the existing robust loss functions and introduce an innovative robust loss function known as the Robust Focal Loss, specifically crafted to address the class imbalance.

Furthermore, Robust-GBDT offers a user-friendly implementation approach, allowing it to effortlessly merge with existing open-source code. This integration is achieved by simply replacing the objective function with RFL. Consequently, Robust-GBDT leverages the computational efficiency and effective handling of complex datasets that are inherent in GBDT models.

In summary, this paper paves the way for the use of nonconvex loss functions within the second-order GBDT framework. This innovation results in the creation of Robust-GBDT, a robust boosting model that enhances the classification accuracy and resilience to label noise, especially in multi-class scenarios. Additionally, Robust-GBDT remains versatile, user-friendly, and computationally efficient, making it an ideal choice for tackling challenging real-world classification problems.


\appendix

\section{Reduction of \textit{gain}}
\label{a.existence}

Let $G = \sum_{i\in I}g_{i}$ denote the sum of the gradient in a leaf, and $H = \sum_{i\in I}h_{i}$ represent the sum of Hessian. 
For a fixed tree structure, we randomly set a few Hessian values to be negative and use $\hat{H}$ to stand for the new sum. 
Suppose that $G_L=\mu G$, $H_L=\nu H$, $\hat{H}=\theta H$, $\hat{H}_L=\tau \hat{H}$. Here, $\mu$, $\nu$, $\theta$, and $\tau$ are all independent and lie within the interval $(0, 1)$.

The new loss reduction is
\begin{equation}
\begin{aligned}
  gain &= (\frac{G_L^2}{\hat{H}_L+\lambda}+\frac{G_R^2}{\hat{H}_R+\lambda}-\frac{G^2}{\hat{H}+\lambda}) \\
    &= G^2(\frac{\mu^2}{\tau\theta H+\lambda}+\frac{(1-\mu)^2}{(1-\tau)\theta H+\lambda}-\frac{1}{\theta H+\lambda}).
\end{aligned}
\end{equation}

Since the tree structure is fixed, $\mu$ and $\nu$ are constants. $\tau$ and $\theta$ are free variables and are determined by the negative Hessian values. When $\tau$=$\mu$, $gain$ becomes 0, leading to the cessation of leaf splitting.

\section{Datasets Description}
\label{a.data}
Table. \ref{T.binary_dataset} and Table. \ref{T.multi_dataset} list the dataset used in this paper.
\#Sample means the instance number of the dataset, \#Feature represents the feature number, \#Classes denotes the class number. IR stands for imbalanced ratio, which equals the positive sample number divided by the number of negative samples.
\%Majority is the percentage of the majority class. Dataset CR has missing values, while others do not.

For binary datasets, labels of the minority class $C_{min}$ in the training dataset are flipped firstly according to a given noise rate $\gamma$, which means $\gamma \cdot \#C_{min}$ of labels are flipped into majority class $C_{maj}$. Then, labels of the majority class in the training data set are flipped in the same number $\gamma \cdot \#C_{min}$ into the minority class $C_{min}$.

For multi-class datasets, there are various methods to introduce noise. In this paper, we flip the labels accordingly to a \textit{pair-flipping matrix} $P$, which is defined as follows:
\begin{equation*}
    P(i, j) = 
    \begin{cases}
    1-\gamma, & i=j,\\
    \gamma, & i+1=j, \\
    0, & otherwise,
    \end{cases}
\end{equation*}
where $P(i, j)$ means the probability of flipping $i$ to $j$.

\begin{table*}[ht]
\normalsize
\renewcommand\arraystretch{1.1}
\centering
\begin{adjustbox}{width=0.9\textwidth}
\begin{tabular}{lccc}
    \toprule[1pt]
        Dataset & \#Sample & \#Feature  & IR\\
    \hline
        abalone\_19 (AB) & 4177  & 10  & 129.53\\
    \hline
        adult (AD) & 48842  & 14  & 3.18\\
    \hline
        car\_eval\_4 (CA) & 1728  & 21  & 25.58\\
    \hline
        credit-approval (CR) & 690  & 15  & 1.25\\
    \hline
        mammography (MA) & 11183  & 6  & 42.01\\
    \hline
        libras\_move(LI) & 360  &90  & 14.00\\
    \hline
        oil (OI) &937 & 49 & 21.85\\
    \hline
        ozone\_level (OZ) &2536  &72  & 33.74\\
    \hline
        protein\_homo (PR) & 145751  &74  & 111.46\\
    \hline
        SantanderCustomerSatisfaction (SA) &200000  & 200  & 8.95\\
    \hline
        sick\_euthyroid (SI) & 3163  & 42  & 9.80\\
    \hline
        solar\_flare\_m0 (SO) & 1389  &32 & 19.43\\
    \hline
        thyroid\_sick (TH) & 3772  & 52 & 15.33\\
    \hline
        yeast\_me2 (YE) & 1484  & 8  & 28.10\\
    \hline
        yeast\_ml8 (YL) &2417  &103  & 12.58\\
    \bottomrule[1pt]
\end{tabular}
\end{adjustbox}
\caption{Binary classification datasets.}
\label{T.binary_dataset}
\end{table*}

\begin{table*}[!ht]
\normalsize
\renewcommand\arraystretch{1.1}
\centering
\begin{adjustbox}{width=0.9\textwidth}
\begin{tabular}{lcccc}
    \toprule[1pt]
        Dataset & \#Sample & \#Feature & \#Classes  & \%Majority\\
    \hline
        balance (BA) & 625  &4 & 3 & 46.08\% \\
    \hline
        car (CA) &1728  & 21 & 4 & 70.02\% \\
    \hline
        cleveland (CL) &297  & 13 & 5 & 53.87\% \\
    \hline
        contraceptive (CO) &1473  &9 & 3 & 42.70\%\\
    \hline
        covertype (CV) &581012  &54 & 7 & 48.76\%\\
    \hline
        flare (FL) & 1066  & 19 & 6 & 31.05\%\\
    \hline
        gas-drift (GA) & 13910  & 128 & 6 & 21.63\%\\
    \hline
        letter (LE) & 20000  & 16  & 26 & 4.06\%\\
    \hline
        texture (TE) & 5500  & 40 & 11 & 9.09\%\\
    \hline
        waveform (WA) & 5000  &40 & 3 & 33.84\%\\
    \bottomrule[1pt]
\end{tabular}
\end{adjustbox}
\caption{Multi-class classification datasets.}
\label{T.multi_dataset}
\end{table*}






\bibliographystyle{elsarticle-num} 
\bibliography{references}

\end{document}